\documentclass[acmsmall]{acmart}
\AtBeginDocument{%
  \providecommand\BibTeX{{%
    \normalfont B\kern-0.5em{\scshape i\kern-0.25em b}\kern-0.8em\TeX}}}

\setcopyright{acmcopyright}
\copyrightyear{2018}
\acmYear{2018}
\acmDOI{XXXXXXX.XXXXXXX}

\usepackage [
    n, % or lambda 
    advantage,
    operators,
    sets,
    adversary,
    landau,
    probability,
    notions,
    logic,
    ff, 
    mm,
    primitives,
    events,
    complexity,
    oracles,
    asymptotics,
    keys
]{cryptocode}

\usepackage{multicol}

\acmJournal{JACM}
\acmVolume{37}
\acmNumber{4}
\acmArticle{111}
\acmMonth{8}

\begin{document}

\title{Recoverable Privacy-Preserving Image Classification through Noise-like Adversarial Examples}

\author{Jun Liu}
\email{yc07453@umac.mo}
%\orcid{0000-0003-1167-5727}
\affiliation{%
  \institution{Stake Key Laboratory of Internet of Things for Smart City, 
 Department of Computer and Information Science, Faculty of Science and Technology, University of Macau}
  %\streetaddress{P.O. Box 1212}
  \city{Macau}
  \country{China}}

\author{Jiantao Zhou}
\authornote{Corresponding author}

\affiliation{%
  \institution{Stake Key Laboratory of Internet of Things for Smart City 
 Department of Computer and Information Science, Faculty of Science and Technology, University of Macau}
  %\streetaddress{P.O. Box 1212}
  \city{Macau}
  \country{China}}
  \email{jtzhou@umac.mo}

\author{Jinyu Tian}
\affiliation{%
  \institution{Faculty of Innovation Engineering, Macau University of Science and Technology}
  %\streetaddress{1 Th{\o}rv{\"a}ld Circle}
  \city{Macau}
  \country{China}}
\email{jytian@must.edu.mo}

\author{Weiwei Sun}
\affiliation{%
  \institution{Alibaba Group}
  %\streetaddress{1 Th{\o}rv{\"a}ld Circle}
  \city{Hangzhou}
  \country{China}}
\email{sunweiwei.sww@alibaba-inc.com}

\renewcommand{\shortauthors}{Jun L et al.}
\newcommand\Tstrut{\rule{0pt}{2.6ex}}

\begin{abstract}
 
 With the increasing prevalence of cloud computing platforms, ensuring data privacy during the cloud-based image related services such as classification has become crucial. In this study, we propose a novel privacy-preserving image classification scheme that enables the direct application of classifiers trained in the plaintext domain to classify encrypted images, without the need of retraining a dedicated classifier. Moreover, encrypted images can be decrypted back into their original form with high fidelity (recoverable) using a secret key. Specifically, our proposed scheme involves utilizing a feature extractor and an encoder to mask the plaintext image through a newly designed Noise-like Adversarial Example (NAE). Such an NAE not only introduces a noise-like visual appearance to the encrypted image but also compels the target classifier to predict the ciphertext as the same label as the original plaintext image. At the decoding phase, we adopt a Symmetric Residual Learning (SRL) framework for restoring the plaintext image with minimal degradation. Extensive experiments demonstrate that 1) the classification accuracy of the classifier trained in the plaintext domain remains the same in both the ciphertext and plaintext domains; 2) the encrypted images can be recovered into their original form with an average PSNR of up to 51+ dB for the SVHN dataset and 48+ dB for the VGGFace2 dataset; 3) our system exhibits satisfactory generalization capability on the encryption, decryption and classification tasks across datasets that are different from the training one; and 4) a high-level of security is achieved against three potential threat models. The code is available at \url{https://github.com/csjunjun/RIC.git}.  
\end{abstract}




\maketitle

\section{Introduction}\label{sec:intro}
                 
 In recent years, machine learning techniques have found widespread applications in various image processing tasks, including classification, segmentation, denoising, etc. \cite{liu2022convnet,10.1145/3578518,10.1145/3231742,10.1145/3446618,Neshatavar2022CVFSIDCM,10.1145/3566125}. Concurrently, the emergence of Machine Learning as a Service (MLaaS) has introduced a new service paradigm, wherein machine learning services are provided to clients through cloud infrastructures \cite{7424435,9384314}. However, the use of MLaaS raises concerns about the data privacy, as client data uploaded to the cloud may contain sensitive information. To address this issue, various Privacy-Preserving Machine Learning (PPML) solutions have been proposed, which aim to perform machine learning tasks without exposing the original private data \cite{8677282,Mohassel2017SecureML,9757847}. Over the past decade, many PPML techniques have been developed to mitigate privacy risks in image classification, including Homomorphic Encryption (HE) \cite{gentry2009fully,gilad2016cryptonets,sanyal2018tapas,9920289}, Multi-Party Computation (MPC) \cite{yao1982protocols,liu2017oblivious,juvekar2018gazelle,bian2020ensei,10089424}, differential privacy \cite{dwork2008differential,abadi2016deep,ji2022privacy}, frequency domain transformations \cite{bian2020ensei,mi2022duetface}, and federated learning \cite{mcmahan2017communication,yang2019federated,10.1145/3426474,10.1145/3537899}, etc.

In the context of privacy-preserving image classification, traditional PPML methods often rely on specially trained classifiers that come with many limitations. For instance, CryptoNets \cite{gilad2016cryptonets} can only classify a certain type of encrypted data; but would fail for plaintext ones. These methods modify components in a standard Convolutional Neural Network (CNN) that are not compatible with HE, such as max pooling and sigmoid activation functions, and then retrain the classifier in the ciphertext domain.  Similarly, TAPAS \cite{sanyal2018tapas} proposes a method to adapt a classifier to encrypted images; but only applicable to binary neural networks \cite{hubara2016binarized}. It should be noted that these specially trained classifiers often have much inferior classification accuracy, compared to the ones trained in the plaintext domain. There are some other drawbacks of traditional PPML methods such as the large overhead of retraining classifiers for each different dataset, difficulty in keeping up with the advancements of state-of-the-art (SOTA) classifiers, high communication costs associated with MPC \cite{shamsabadi2020privedge}, etc.

In this work, we propose a \textbf{R}ecoverable privacy-preserving \textbf{I}mage \textbf{C}lassification (RIC) system, aiming to alleviate the above limitations from three perspectives: 1) enable any plaintext-domain classifiers to classify encrypted images, without the need of retraining dedicated classifiers and sacrificing classification accuracy; 2) allow high-fidelity recovery (recoverable) of the original plaintext image by using a secret key; and 3) possess desired generalization capability to unknown datasets. The application scenarios of our proposed RIC are illustrated in Fig.\ref{fig:story}. For instance, consider a government department that needs to store a large number of driving license images in the cloud. Since these images contain sensitive information such as users' photos and identification numbers, it is imperative that they are stored in an encrypted format. Meanwhile, for various purposes, e.g., efficient management of image files, it is much desirable to perform image classification over these encrypted images. Additionally, the government department should have the capability to retrieve the encrypted images and decrypt them into their original form with high fidelity, with the assistance of a secret key. Another example is regarding the face images captured by surveillance cameras and stored in the cloud. These images, privacy-sensitive in nature, play a vital role in many critical situations, e.g., tracking the appearance of criminals. Again, classification has to be carried out in the encrypted domain, as face images are encrypted and stored in the cloud.  Moreover, in this application, for confirming the identity of criminals, the ciphertext needs to be decrypted with a small degradation.           

In our RIC system, there are three parties involved: data providers, the cloud, and authorized users, as shown in Fig.\ref{fig:story}. Specifically, data providers upload encrypted images to the cloud using the encryption functionality provided by the RIC system. The cloud can then perform classification directly on the encrypted data using classifiers trained in the plaintext domain. Also, authorized users can decrypt the encrypted images using a secret key shared through a private channel. Technically, our proposed RIC system consists of three main components: an encoder, a decoder, and a classifier.  The encoder and decoder need to be trained to perform the encryption and decryption respectively, while the classifier is pre-trained in the plaintext-domain serving as an assisting module. In the process of encryption, the encoder is used to conceal a plaintext image by applying our carefully designed Noise-like Adversarial Example (NAE). The NAE is generated by optimizing a Random Noisy Image (RNI) in such a way that it is classified as the same class label as the given plaintext image. The image encoded with an NAE can then be directly classified by the given plaintext-domain classifier, without requiring a dedicated re-trained classifier. Such a flexibility allows for an easy integration with powerful SOTA classifiers. During the decryption (recovery), the decoder uses the same NAE to reconstruct the plaintext image with a high fidelity, through a Symmetric Residual Learning (SRL) architecture.

The contributions of this work are summarized as follows:
 \begin{itemize}
 		\item We propose a novel system for privacy-preserving image classification in the cloud. This system does not require to re-train a dedicated classifier in the encrypted domain and does not compromise the classification accuracy. Such a property allows to incorporate any existing powerful classifiers trained in the plaintext domain.

		\item We design a secure encoding-decoding mechanism in RIC for protecting the privacy of images and recovering them faithfully with a secret key, through our proposed NAE and SRL architecture. 
		
		\item Extensive experiments demonstrate that RIC outperforms SOTA PPML methods in terms of classification accuracy and recovery fidelity. Furthermore, RIC shows good generalization capability across different datasets. It is also validated that a high-level of security against three potential threat models is achieved.

 \end{itemize}

The organization of the remaining paper is as follows. We first review related works in Section \ref{section:rewo} and then define the design goals, system model, and threat models in Section \ref{sec:sec3}. Subsequently, we give the details of the most crucial modules of our RIC in Section \ref{sec:sec4}. After that, extensive experimental results are provided in Section \ref{sec:sec5}, and the security analysis is presented in Section \ref{sec:sec6}. Finally, Section \ref{sec:sec7} concludes.

 \begin{center}
	\begin{figure}[t]		\centering{\includegraphics[width=0.8\textwidth]{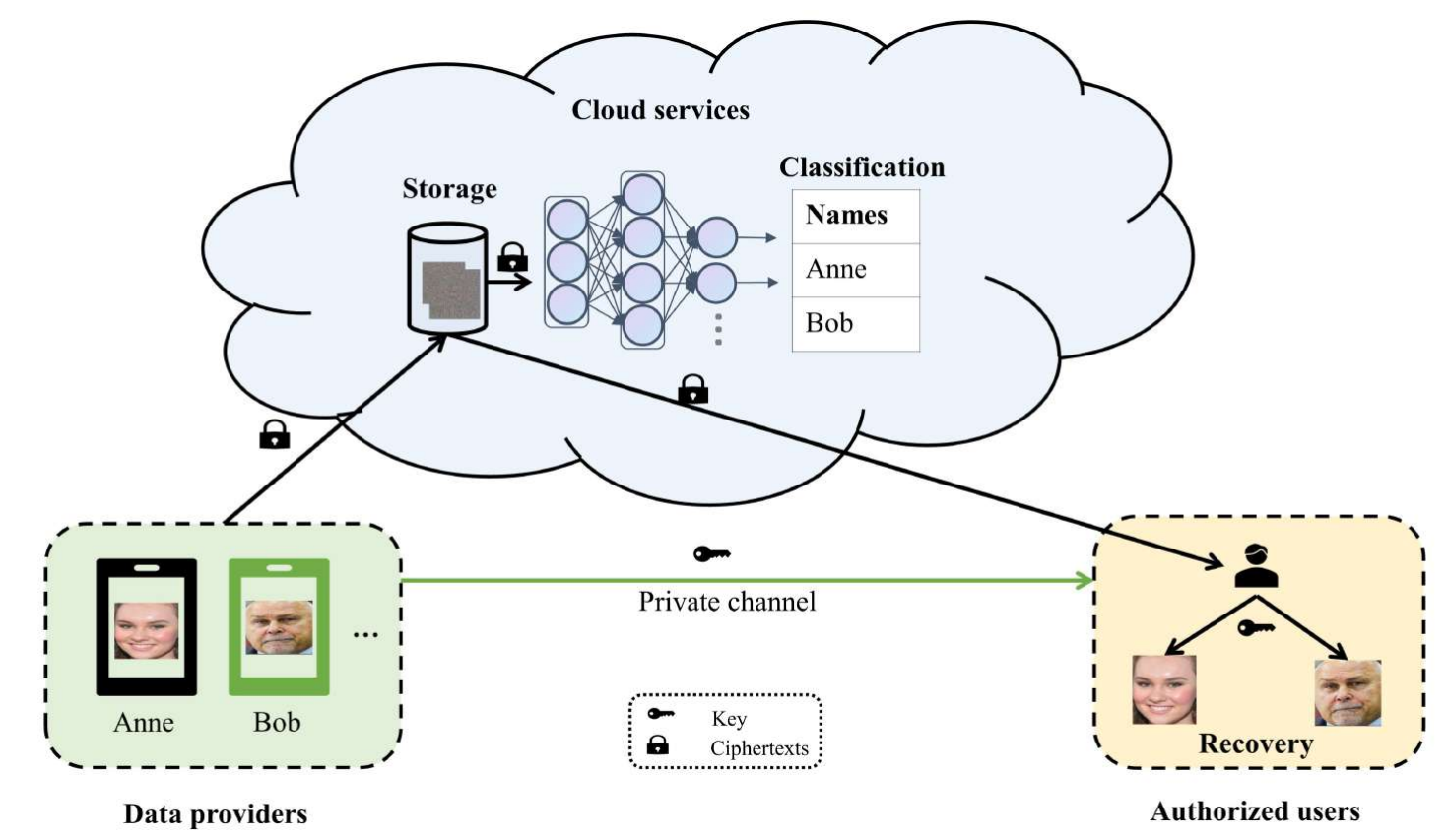}}
		\caption{Schematic diagram of our proposed recoverable privacy-preserving image classification system.}
		\label{fig:story}
	\end{figure}
\end{center}

	\section{Related Works}\label{section:rewo}
 In this section, we review the related works on privacy-preserving image classification and adversarial attack.

\subsection{Privacy-Preserving Image Classification}

Cryptography techniques, such as HE\cite{rivest1978data} and Secure MPC\cite{goldreich1998secure}, have been applied to address the privacy-preserving classification problem. Specifically, HE transforms a plaintext to a ciphertext that can be computed mathematically, enabling some computations in the encrypted domain. One notable method is TAPAS \cite{sanyal2018tapas}, which utilizes a fully HE scheme \cite{gentry2009fully} that supports operations on binary data. TAPAS employs specially designed binary neural networks to perform inference on encrypted data. However, HE-based methods often come with huge computational overhead due to the homomorphic computations involved.   Additionally, they may suffer from the accuracy degradation because of the approximated functions and binarized parameters in the classifiers \cite{gilad2016cryptonets, sanyal2018tapas}. Another approach is taken by a branch of methods that employ Secure MPC and define interactive prediction protocols. Representatives of such methods include MiniONN \cite{liu2017oblivious}, Gazzle \cite{juvekar2018gazelle}, and ENSEI \cite{bian2020ensei}. Unfortunately, MPC-based methods still face various challenges, including the time-consuming training of dedicated classifiers \cite{shamsabadi2020privedge} and high communication overheads. For instance, MiniONN exhibits a communication cost of over 9 minutes and requires message sizes of 9,000 MB when performing inference on images from Cifar-10\cite{krizhevsky2009learning} dataset with a size of $32\times32\times3$. Again, using oblivious approximated network components inevitably leads to accuracy drops, which sometimes are significant. 

Another line of research involves utilizing features extracted by CNNs as encrypted images, without explicitly relying on cryptography. However, these methods still require uniquely designed and trained classifiers to classify the encrypted data. Methods such as DAO\cite{dong2017dropping}, Hybrid\cite{osia2020hybrid}, FEAT\cite{ding2020privacy} and ARDEN\cite{wang2018not} assign initial layers of certain CNNs to the users' side, while transmitting output features as the encrypted data to the cloud server side for the inference. Some other methods, e.g. DPFE\cite{osia2020deep} and Cloak\cite{mireshghallah2021not}, intentionally decrease the mutual information between the extracted feature, i.e., the ciphertext, and the plaintext to achieve the privacy-preserving ability. Furthermore, InstaHide\cite{huang2020instahide} encrypts a plaintext image by combining it with a set of randomly chosen images and applying a random pixel-wise mask. Subsequently, a classifier is trained on this processed dataset. It is worth mentioning that InstaHide could achieve a satisfactory level of privacy only if the additional dataset used for mixing the original dataset is sufficiently diverse and private.      
                 
	\subsection{Adversarial Attacks against Classifiers }
Adversarial attacks against classifiers aim to generate Adversarial Examples (AE) that deceive the classifiers into predicting incorrect labels. These attacks can be categorized into targeted and non-targeted, based on whether the incorrect label is pre-specified. Also, attack methods can be divided into white-box, black-box, or even grey-box, depending on the knowledge accessible to the attacker. Here, we focus on white-box adversarial attacks, which are most relevant to the current work. Among white-box approaches, there are two main streams: gradient projection methods and Lagrange-form methods, along with a few other attack techniques.     

Specifically, Gradient projection attacks, such as Fast Gradient Sign Method (FGSM)\cite{goodfellow2014explaining}, Projected Gradient Descent (PGD)\cite{madry2017towards}, and Momentum Iterative
Fast Gradient Sign Method (MIFGSM)\cite{dong2018boosting}, endeavor to seek a feasible AE by minimizing a Cross-Entropy (CE) loss function under a distortion budget. The gradient descent with respect to this CE loss is used for searching a local minimum, and projection operations control the perturbation distance in the Euclidean space. Alternatively, Lagrange-form attacks tackle the classification and the distortion simultaneously by introducing both of their constraints into the optimization objective function, potentially achieving a higher attack success rate and smaller adversarial perturbations. Among Lagrange-form attacks, C\&W \cite{carlini2017towards} is one of the most powerful methods, which proposes to integrate a margin loss for mis-classification in the optimization objective. Some other methods also attempt to find AEs from a geometric perspective. For example, DeepFool\cite{moosavi2016deepfool} calculates the normal vectors of the classification decision boundary so as to approach the AE with a minimal perturbation. Another work named Walking on the Edge (WE) \cite{zhang2020walking} optimizes the perturbation magnitude by iteratively projecting the gradient of the distortion on the tangent space of the manifold of the decision boundary.  

It should also be noted that these adversarial attacks typically initiate the optimization process from a natural image. To the best of our knowledge, there has been no exploration of the optimization from a purely random image to deceive the classifier. As will be clear soon, such generated NAEs would play an important role in our proposed privacy-preserving image classification approach.

\section{Design goals, System Model, and threat models}\label{sec:sec3}

This section presents an overview of our proposed privacy-preserving image classification system. We begin by outlining the design goals that guide the development of our system. Next, we describe the system model, which encompasses the key components and their interactions. We also discuss three potential threat models that could be used to assist the security evaluation.

\subsection{Design Goals}
Our proposed system aims to achieve the following design goals:
\begin{itemize}
\item\textbf{Accuracy}: The classification accuracy on the encrypted data would be as high as the one on original non-encrypted data resulting from a SOTA classifier trained in the plaintext domain.  
    \item\textbf{Confidentiality}: The meaningful information contained within the encrypted image should be inaccessible to adversaries. Only authorized parties would be able to recover the original image from the ciphertext with a secret key.
    \item\textbf{Generalization}: Once our system is well-trained on a particular dataset, it can be directly used for encrypting and decrypting other unknown datasets. The resulting ciphertexts can also be accurately classified by the corresponding classifiers trained in the plaintext domain.  
\end{itemize}
	
\subsection{System Model}\label{sec:systemmodel}
Upon defining the design goals, we construct our proposed RIC system as illustrated in Fig.\ref{fig:pics}. As mentioned earlier, the system involves three key participants: data providers, authorized users, and a cloud server. It also incorporates the following three important modules:
\begin{itemize}
      \item \textbf{Privacy-protection Module}: $\mathcal{E}(\mathbf{x_f},\mathbf{n_{adv}})\rightarrow\mathbf{x^\prime_q}$. The data provider employs a locally deployed encoder $\mathcal{E}$ to encrypt the feature $\mathbf{x_f}=\mathcal{F}(\mathbf{x})$ of a given plaintext image $\mathbf{x}$, where a feature extractor $\mathcal{F}$ is applied. This privacy-protection module involves the generation of an RNI $\mathbf{n}$ using a Pseudo Random Number Generator (PRNG) with a secret key $k$. The $\mathbf{n}$ is then subjected to an adversarial attack method $\Phi$, resulting in our designed NAE, denoted as $\mathbf{n_{adv}}$.  Subsequently, the encoder $\mathcal{E}$ encrypts the feature $\mathbf{x_f}$ with the assistance of $\mathbf{n_{adv}}$, producing the encrypted image $\mathbf{x^\prime_q}$. Essentially, the image $\mathbf{x^\prime_q}$ not only ensures privacy preservation but also retains the class label $\hat{y}$ of $\mathbf{x}$ assigned by the classifier $\mathcal{C}$. Note that we only consider the cases where the classifier achieves accurate classification.          
    \item \textbf{Image Classification Module}: $\mathcal{C}(\mathbf{x^\prime_q})\rightarrow\hat{y}$. The cloud server utilizes the identical plaintext-domain classifier $\mathcal{C}$ to conduct inference on the protected image $\mathbf{x^\prime_q}$. Denote $\hat{y}$ as the predicted class label. A very desirable property of our design is that $\mathcal{C}$ can be any well-trained SOTA classifier, avoiding the necessity of a dedicated classifier in the encrypted domain.
    \item \textbf{Image Recovery Module}: $\mathcal{D}(\mathbf{x^\prime_q},\mathbf{n_{adv}})\rightarrow\tilde{\mathbf{x}}$. A plaintext image $\tilde{\mathbf{x}}$ can be reconstructed from the encrypted image $\mathbf{x^\prime_q}$ using the decoder $\mathcal{D}$ and the same $\mathbf{n_{adv}}$ employed in the privacy-protection module. The discussion on $\mathbf{x^\prime_q}$ is deferred to Section \ref{sec:pripromodule}. With the pre-negotiated key $k$ shared through a private channel and applying the same adversarial attack method $\Phi$, we ensure the generation of the identical $\mathbf{n_{adv}}$. Such a recovery process takes place on the authorized users' local devices, ensuring that the recovered images are not exposed to the cloud server.
    
\end{itemize}

\subsection{Threat Models}\label{sec:threatmodel}
Considering the aforementioned design goals, we identify three potential threat models that could pose malicious behaviors against our proposed system. In Section \ref{sec:sec6}, we will evaluate the security under these threat models.
\begin{itemize}
    
    \item\textbf{Brute-force Attack}: Adversaries who do not have authorization may attempt to recover the plaintext image through a brute-force attack. A brute-force attack involves exhaustively traversing a wide range of keys in the image recovery module until a comprehensible plaintext is obtained. 
\item\textbf{Cloud Attack}: An honest but curious cloud server, who has access to various ciphertexts and the corresponding classifiers, may attempt to recover the plaintext images from the ciphertexts with the help of their classifiers. 
 \item\textbf{Known-plaintext Attack}: Malicious users, who have authorized access to certain historical ciphertexts and their corresponding restored images, attempt to recover plaintext images from encrypted ones beyond their access rights.

\end{itemize}

\begin{center}
\begin{figure*}[t!]			\centering{\includegraphics[width=1\textwidth]%
{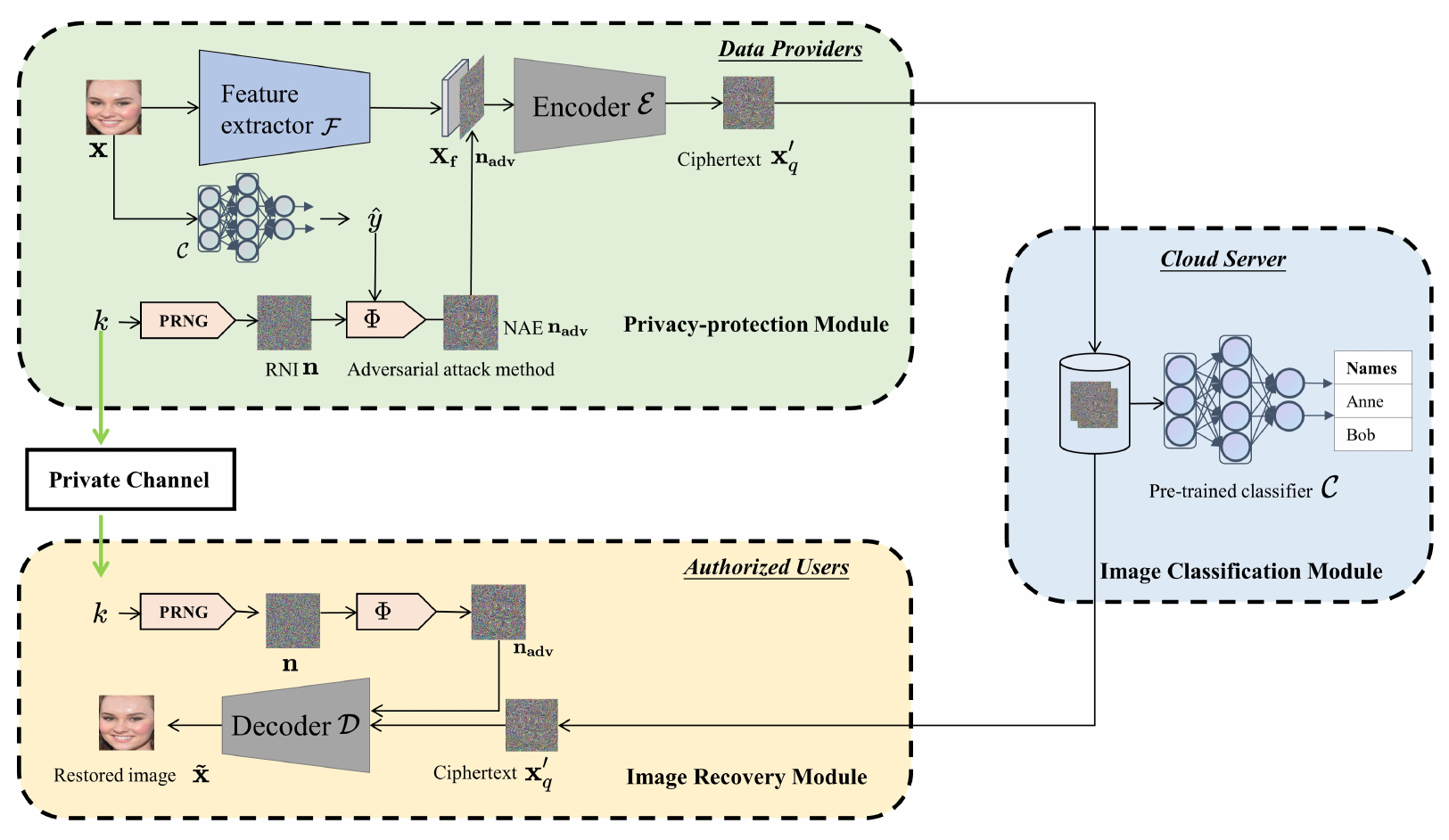}}
			\caption{The system model and architecture of our proposed RIC system.}
			\label{fig:pics}
		\end{figure*}
	\end{center}
 
\section{The proposed RIC framework}\label{sec:sec4}
In this section, we provide a detailed description of the key modules briefly introduced in Section \ref{sec:systemmodel}, as well as the training scheme for the entire system.         

\subsection{Privacy-protection Module}\label{sec:pripromodule}
With regard to the design goal in terms of accuracy and confidentiality, the resulting privacy-protected image $\mathbf{x^\prime_q}$ needs to be classified by the classifier $\mathcal{C}$ as the same label of the plaintext image $\mathbf{x}$. Our objective in the privacy-protection module is to transform $\mathbf{x}$ into an image that exhibits a noisy texture for protecting its privacy, while still being classified correctly as $\hat{y}$ by $\mathcal{C}$.

Foremost, we perform a targeted adversarial attack on an RNI, which is sampled randomly from a uniform distribution, to generate an NAE. After a successful adversarial attack procedure, the resulting NAE can be classified as an arbitrarily pre-defined target class label by the classifier, while preserving its noisy texture. Specifically, to initiate the attack, we employ a PRNG that uses the secret key $k$ to generate an RNI, denoted as $\mathbf{n}$, by sampling values from the uniform distribution $U[0,1]$ in $\mathbb{R}^{d}$, where $d$ represents the dimension of the original image $\mathbf{x}$. Mathematically, the generation of $\mathbf{n}$ can be formulated as: 
 	\begin{equation}\label{eq:n1}
		\begin{aligned}
			\mathbf{n}=\mathop{PRNG}(k,U).
		\end{aligned}
	\end{equation}
Subsequently, the corresponding NAE $\mathbf{n_{adv}}$ can be calculated by   

\begin{equation}\label{eq:nadv}
		\begin{aligned}
	\mathbf{n_{adv}}=\mathbf{n}+\Delta,
		\end{aligned}
	\end{equation} where $\Delta$ is obtained through a targeted attack on $\mathbf{n}$. Specifically, we solve the following optimization problem by minimizing the widely-adopted margin loss \cite{carlini2017towards}:      

\begin{equation}\label{eq:clsnadv}
\begin{gathered}
\mathop{\min}_{\Delta}\mathop{\max}\left\{\mathop{\max}_{i\neq\hat{y}}\mathcal{C}(\mathbf{n}+\Delta)[i]-\mathcal{C}(\mathbf{n}+\Delta)[\hat{y}],-g_1\right\}\\
s.t. \ \mathbf{n}+\Delta\in[0,1]^d,
\end{gathered}
\end{equation}
where the prediction class label $\hat{y}=\mathop{\arg\max}_{i}\mathcal{C}(\mathbf{x})[i]$, $\mathcal{C}(\cdot)=[c_0,c_1,...,c_{v-1}]$ represents the logits output of the penultimate layer of $\mathcal{C}$ and $\mathcal{C}(\cdot)[i]$ stands for the score of the $i$-th class. The parameter $g_1\geq0$ indicates the adversarial extent of the resulting $\Delta$. The constraint in (\ref{eq:clsnadv}) ensures that the perturbed image remains a valid digital image. 

It should be noted that there must exist a solution $\Delta$ to the optimization problem (\ref{eq:clsnadv}), regardless of the original predicted class of $\mathbf{n}$ and the target class $\hat{y}$. In particular, the $d$-dimensional space, where the data distribution is located, is divided into multiple polyhedra $P_j$ by the decision boundary of the classifier\cite{moosavi2016deepfool}. Each polyhedron $P_j$ contains points that belong to the same class $j$. If there exists a point $\mathbf{o} \in P_{\hat{y}}$, then we can naturally obtain a vector $\Delta$ by calculating $\mathbf{o} - \mathbf{n}$. Therefore, there must be a $\Delta = \mathbf{o} - \mathbf{n}$ that satisfies $\mathop{\arg\max}_{i}\mathcal{C}(\mathbf{n} + \Delta)[i] = \mathop{\arg\max}_{i}\mathcal{C}(\mathbf{o})[i]=\hat{y}$.  

To explicitly find the perturbation $\Delta$, we utilize the Adam optimizer \cite{kingma2014adam} to solve the problem (\ref{eq:clsnadv}). For simplicity, we express the process of generating $\Delta$ as:
\begin{equation}\label{eq:nae}
		\begin{aligned}			\Delta=\Phi(\mathbf{n}),
		\end{aligned}
	\end{equation}
where $\Phi$ encapsulates the procedure for obtaining $\Delta$ as an adversarial attack method. Here, we select the margin loss and Adam because they are used in the C\&W attack, which shows a higher success rate compared to many gradient projection approaches integrated with CE loss. It is worth mentioning that our scheme is flexible to be incorporated with more powerful loss functions and more effective solvers when they are available.    

Upon obtaining $\mathbf{n_{adv}}$, we employ it to safeguard the privacy information of the plaintext image $\mathbf{x}$. The specific procedure is depicted in Fig.\ref{fig:encoder}(a), where the plaintext image is initially fed into a CNN-based feature extractor $\mathcal{F}$ to acquire the image feature $\mathbf{x_f}$, enabling image encryption in the feature domain. Subsequently, $\mathbf{n_{adv}}$ and $\mathbf{x_f}$ are concatenated and serve as the input to the encoder. The encoder consists of three stacks of convolutions with kernels of varying sizes, allowing for the encoding of the input at different receptive fields. This encoder fuses the information of $\mathbf{n_{adv}}$ and $\mathbf{x_f}$, and eventually generates a preliminary privacy-protected image $\mathbf{u_{en}}$. In order to facilitate $\mathbf{u_{en}}$ being recognized as belonging to the same class as $\mathbf{x}$, thereby ensuring classification accuracy, we introduce a shortcut from $\mathbf{n_{adv}}$ to $\mathbf{u_{en}}$, as indicated by the red line in Fig.\ref{fig:encoder}(a). Essentially, this shortcut incorporates the category information of $\mathbf{n_{adv}}$ into $\mathbf{u_{en}}$ such that $\mathbf{u_{en}}$ gets closer to the class of $\mathbf{n_{adv}}$, namely that of $\mathbf{x}$. Meanwhile, it directly leverages the noise appearance of $\mathbf{n_{adv}}$ to further enhance the privacy protection capability of $\mathbf{u_{en}}$.  By blending $\mathbf{n_{adv}}$ and $\mathbf{u_{en}}$ in a weighted manner, the preliminary privacy-protected image $\mathbf{x^\prime}$ is generated as:

 	\begin{equation}\label{eq:en1}
  \begin{aligned}
			\mathbf{x^\prime} &= (1-\lambda)\cdot\mathbf{u_{en}} + \lambda\cdot\mathbf{n_{adv}}\\
   	&=(1-\lambda)\cdot\mathcal{E}\left(\mathcal{F}(\mathbf{x})||\mathbf{n_{adv}}\right)+\lambda\cdot\mathbf{n_{adv}},
    \end{aligned}
	\end{equation}
 where || denotes a channel-wise concatenation operation, and $\lambda\in(0,1)$ balances the relative importance of our design goals. More discussions on $\lambda$ are deferred to Section \ref{sec:abation}.    

Subsequently, to make $\mathbf{x^\prime}\in[0,1]^d$ display appropriately, it has to be further quantized into its standard PNG format where pixel values are integers within the range of [0,255]. Then, the quantized ciphertext, upon entering the image recovery module, undergoes an interval transformation from [0,255] back to [0,1] for further processes. In order to mitigate the errors introduced by the rounding operation in the quantization step, we involve the quantization and interval transformation steps, dubbed as $\mathbf{Q}(\cdot)$, during the training process. However, due to the rounding operation, quantization itself is non-differentiable. To ensure the differentiability of the loss function, we employ additive uniform noise on $\mathbf{x}^\prime$, akin to \cite{cheng2021iicnet}. To this end, at the training stage, the eventual ciphertext $\mathbf{x^\prime_q}$ is generated by: 
 \begin{equation}\label{eq:xq}
		\begin{aligned}
		\mathbf{x^\prime_q}&=\mathbf{Q}(\mathbf{x}^\prime)\\
  &=\mathrm{clip}(255\cdot\mathbf{x}^\prime+u,0,255)/255,
		\end{aligned}
	\end{equation}
 where the uniform noise $u\sim U[-0.5,0.5]$.

\begin{center}
		\begin{figure}[t!]
			\centering{\includegraphics[width=0.8\textwidth]{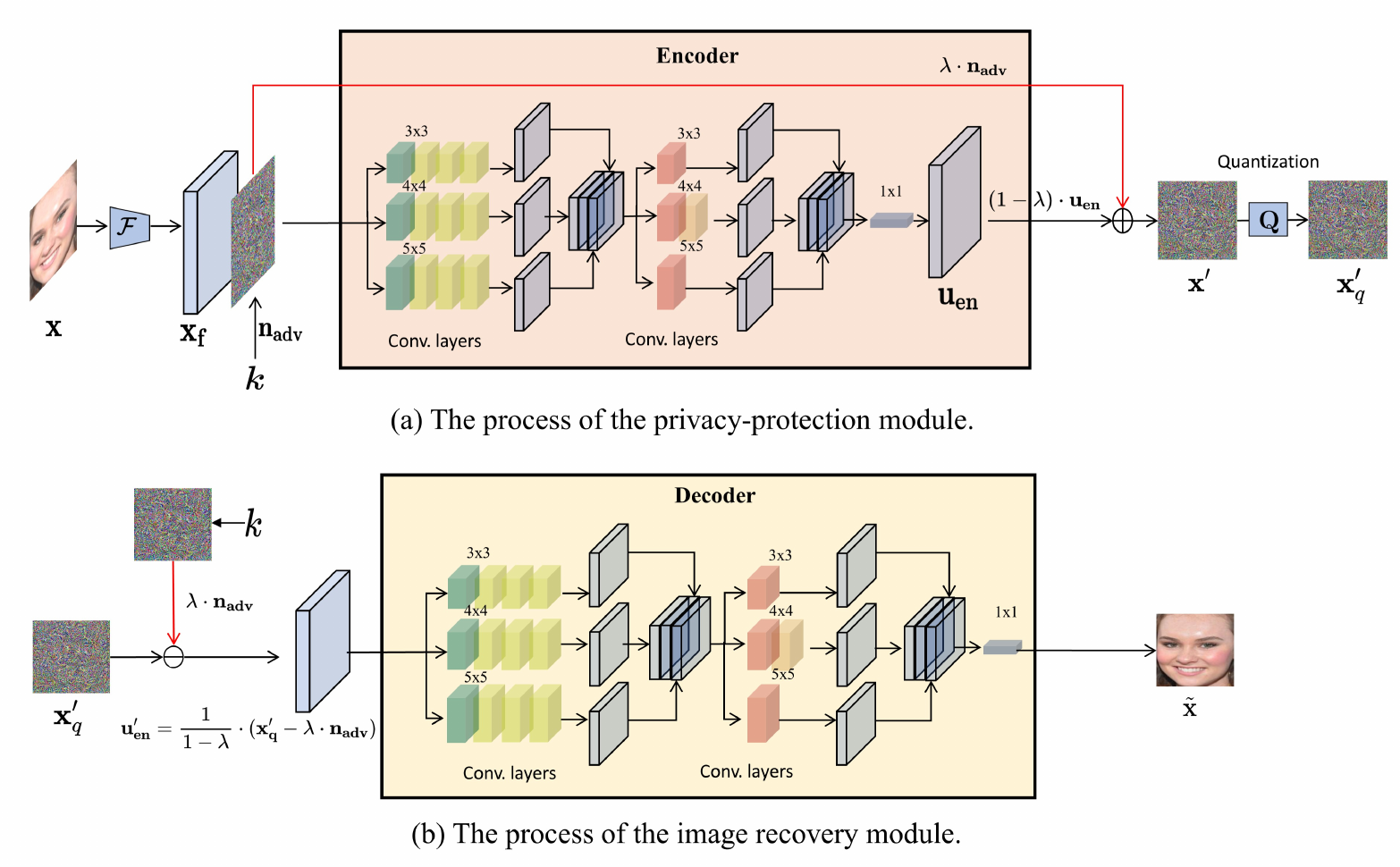}}
			\caption{The architecture of privacy-protection and image recovery modules in our proposed RIC system.}
			\label{fig:encoder}
		\end{figure}
	\end{center} 
 
\subsection{Image Classification Module}
The cloud server can directly utilize the same plaintext-domain classifier $\mathcal{C}$ used in the privacy-protection module to perform classification on the ciphertext. It is worth noting that this classifier can be any well-trained classifier, eliminating the need to train a dedicated classifier specifically for use in the ciphertext domain. Such a property allows us to share the advantages of the constantly improving SOTA classifiers. Mathematically, the classification result of the ciphertext $\mathbf{x^\prime_q}$ can be expressed as: 
\begin{equation}\label{eq:classify}	
		\tilde{y}=\mathop{\arg\max}\limits_{i}\mathcal{C}(\mathbf{x^\prime_q})[i].\\	
	\end{equation} 

As will be validated experimentally later, the predicted class $\tilde{y}$ of the ciphertext remains consistent with the class $\hat{y}$ of the plaintext with a very high probability. This is because our system design ensures that the ciphertext's classification result is predominantly governed by the NAE employed for the encryption, and this NAE's category is intentionally aligned with the original plaintext. In addition, the property that the classifier achieves excellent classification performance in both the plaintext and ciphertext domains also enables the classifier to work seamlessly without distinguishing between the ciphertext domain and the plaintext domain.

\subsection{Image Recovery Module}\label{sec:decryption}
 The image recovery module aims to decode the ciphertext $\mathbf{x^\prime_q}$ into the plaintext image $\tilde{\mathbf{x}}$ that closely resembles the original plaintext image $\mathbf{x}$. In this module, as depicted in Fig.\ref{fig:encoder}(b), the authorized user receives the pre-negotiated key $k$ through a private channel, and then uses it to generate the identical NAE $\mathbf{n_{adv}}$ as the one in the privacy-protection module. The generation of $\mathbf{n_{adv}}$ in the image recovery module can be summarized as follows:
 	\begin{equation}\label{eq:n1de}
   \begin{aligned}
      \begin{cases} 
   \mathbf{n}=\mathop{PRNG}(k,U),\\
   \Delta=\Phi(\mathbf{n}),\\
   \mathbf{n_{adv}} = \mathbf{n}+\Delta.
   \end{cases}
		\end{aligned}
	\end{equation}

 The generated $\mathbf{n_{adv}}$ is then subtracted from the encrypted image $\mathbf{x^\prime_q}$ received from the cloud. This subtraction operation gives rise to an image $\mathbf{u^\prime_{en}}$ that bears an utmost resemblance to the output of the encoder $\mathbf{u_{en}}$. The generation of $\mathbf{u^\prime_{en}}$ can be calculated by:
 \begin{equation}\label{eq:decoder1}
		\begin{aligned}
		\mathbf{u^\prime_{en}}&=\frac{1}{1-\lambda}\cdot(\mathbf{x^{\prime}_q}
			-\lambda\cdot\mathbf{n_{adv}}).
		\end{aligned}
	\end{equation}
The addition and subtraction operations represented by (\ref{eq:en1}) and (\ref{eq:decoder1}), i.e. the read lines in Fig.\ref{fig:encoder}(a) and Fig.\ref{fig:encoder}(b), respectively, collectively constitute the SRL. This SRL ensures that the input of the decoder $\mathbf{u^\prime_{en}}$ closely matches the output of the encoder $\mathbf{u_{en}}$, which facilitates efficient encoding and decoding, leading to improved image reconstruction performance. Further justification of this claim is provided in Section \ref{sec:abation}. 
 
  Finally, we employ a CNN-based decoder to reconstruct $\mathbf{u^\prime_{en}}$ into the plaintext image $\tilde{\mathbf{x}}$ which closely resembles the original image $\mathbf{x}$. As shown in Fig.\ref{fig:encoder}(b), the architecture of this decoder is designed to be symmetrical to that of the encoder. The generation of $\tilde{\mathbf{x}}$ can be formulated as: 
    	\begin{equation}\label{eq:decoder3}
		\begin{aligned}
			\tilde{\mathbf{x}}&=\mathcal{D}(\mathbf{u^\prime_{en}}).	
		\end{aligned}
	\end{equation}

\subsection{Training the RIC System}
 During the training procedure of our RIC model, we focus on optimizing three crucial sub-networks: the feature extractor $\mathcal{F}$, the encoder $\mathcal{E}$, and the decoder $\mathcal{D}$. We use pairs of plaintext images along with their corresponding ground-truth labels $y$ as training data. The objective of the training is to minimize the following loss function $\mathcal{L}$ with respect to the network parameters of $\mathcal{F}$, $\mathcal{E}$ and $\mathcal{D}$: 
	
	\begin{equation}\label{eq:loss2}
		\begin{aligned}		
			\mathcal{L}=\beta\cdot\mathcal{L}_{rec} 
			+\mathcal{L}_{adv},
		\end{aligned}
	\end{equation}
 where $\beta$ is a weighting factor, $\mathcal{L}_{rec}$ represents the reconstruction loss, and $\mathcal{L}_{adv}$ denotes the adversarial loss. Here, the reconstruction loss $\mathcal{L}_{rec}$ is used to measure the recovery error and can be computed using various Euclidean distance metrics. In our approach, we employ the $L_2$ distance:
    \begin{equation}\label{eq:decoder4}
		\begin{aligned}
			\mathcal{L}_{rec} = \parallel\mathbf{x}-\tilde{\mathbf{x}}\parallel_{2}.
		\end{aligned}
	\end{equation}
	To evaluate the attack capability of $\mathbf{x^\prime_q}$, we employ the margin loss\cite{carlini2017towards} for defining $\mathcal{L}_{adv}$, namely, 
    \begin{equation}\label{eq:loss1}
		\begin{aligned}
			\mathcal{L}_{adv}= \mathop{\max}\left\{\mathop{\max}\limits_{i\neq y}\mathcal{C}(\mathbf{x^{\prime}_q})[i]-\mathcal{C}(\mathbf{x^{\prime}_q})[y],  -g_2 \right\},
		\end{aligned}
	\end{equation} where $g_2\geq0$ is an adjustable parameter that controls the adversarial extent of $\mathbf{x^{\prime}_q}$.

\begin{table}[t]
  \centering
  \caption{The classification accuracy ($\%$) comparisons with PPML methods.}
  \scalebox{0.85}{
    \begin{tabular}{c|c|c|c|c|c|c|c|c|c}
    \hline\hline 
    Methods&\multicolumn{3}{c|}{ImageNet} &\multicolumn{3}{c|}{Cifar-10} & \multicolumn{3}{c}{MNIST} \Tstrut\\
\cline{2-10}&$ACC_p$&$ACC_e$          & $\mathop{ADP}\downarrow$&$ACC_p$&$ACC_e$          & $\mathop{ADP}\downarrow$& $ACC_p$&$ACC_e$          & $\mathop{ADP}\downarrow$  \Tstrut\\
    \hline
    RIC (\textbf{Ours}) & 77.80& \textbf{77.80}    & \textbf{0.00}    & 97.09& \textbf{97.09}    & \textbf{0.00}  & 99.75 & \textbf{99.75} & \textbf{0.00}   \Tstrut\\
        \hline  
     InstaHide\cite{huang2020instahide} &77.80&72.90&6.30 & 97.09 & 91.46  &5.80 &99.75& 98.35&1.40\Tstrut \\  \hline  
     Cloak\cite{mireshghallah2021not}  &77.80 &23.19&70.19  &97.09  &60.58  &37.58 & 99.75 &85.42 &14.37 \Tstrut\\
        \hline % 
         TAPAS\cite{sanyal2018tapas}&    77.80   &   7.84 &  89.92   &  97.09&    61.27 &36.89 & 99.75 & 97.97 & 0.78     \Tstrut \\ 
    \hline
    ARDEN\cite{wang2018not} &- &-&-& 97.09 & 88.31 & 9.04& 99.75 & 99.70& 0.05\Tstrut \\
        \hline% 
    MiniONN\cite{liu2017oblivious} &- &-&-& 97.09   & 81.61 & 5.36  & 99.75& 99.17& 0.58   \Tstrut\\
    \hline\hline
    \end{tabular}%
    }
  \label{tab:ppmlacc}%
\end{table}%

 \begin{center}
		\begin{figure}[t]	\centering{\includegraphics[width=1\textwidth]{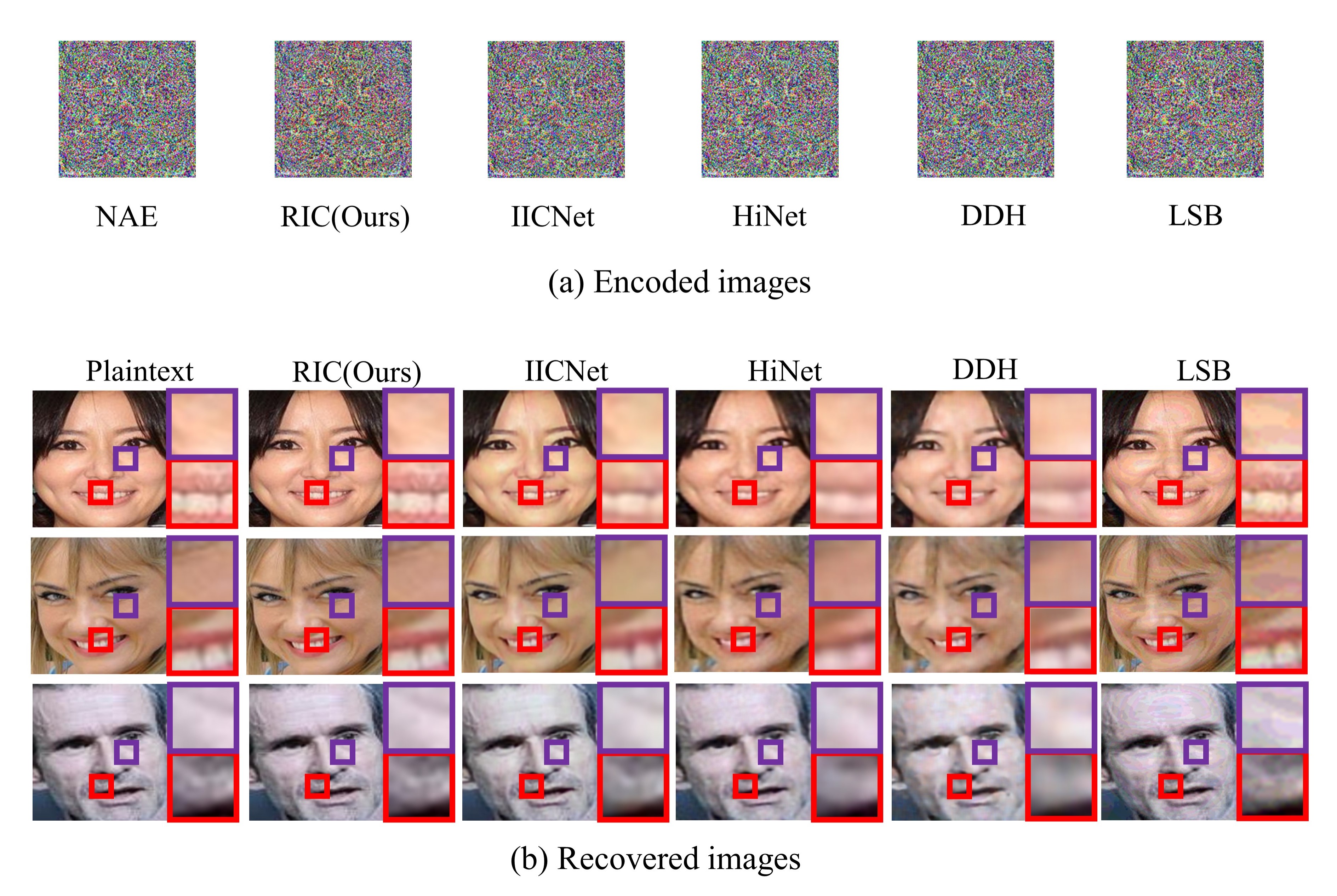}}
			\caption{Visualization of the encoded and recovered images by different methods.}
			\label{fig:samples}
		\end{figure}
	\end{center}

\begin{table*}[t]
  \centering
  \caption{The confidentiality in terms of recoverability performance compared with steganography-like methods.}
  \scalebox{0.85}{
    \begin{tabular}{c|c|c|c|c|c|c}
   \hline\hline
    Datasets & \multicolumn{3}{c|}{SVHN}                             & \multicolumn{3}{c}{VGGFace2}\Tstrut \\
    \hline
    Methods &  PSNR$\uparrow$& SSIM $\uparrow$& LPIPS$\downarrow$ &  PSNR$\uparrow$& SSIM$\uparrow$ & LPIPS$\downarrow$ \Tstrut\\
    \hline
    RIC (\textbf{Ours})  &  \textbf{51.46} & \textbf{0.9995} & \textbf{0.0001}  & \textbf{48.04} & \textbf{0.9972} & \textbf{0.0015}   \Tstrut\\
    \hline
    LSB\cite{lsb1998} w/NAE  & 31.26  & 0.9553  & 0.0271    & 31.80  & 0.8681  & 0.1388   \Tstrut \\
    \hline
    IICNet\cite{cheng2021iicnet} w/NAE & 38.59  & 0.9908  & 0.0058   &  34.06  & 0.9535  & 0.0853 \Tstrut \\
    \hline
    HiNet\cite{jing2021hinet} w/NAE&  34.04  & 0.9744  & 0.0135  & 33.53  & 0.9330  & 0.1206\Tstrut \\
    \hline
    DDH\cite{zhang2020udh} w/NAE&  29.74  & 0.9318  & 0.0259   & 28.29  & 0.8861  & 0.1495  \Tstrut\\
  
    \hline\hline
    \end{tabular}%
    }
  \label{tab:total1}%
\end{table*}%
 
\section{Experimental results}\label{sec:sec5} 
	In this section, we provide experimental results of our proposed RIC and compare it with SOTA methods. We begin by presenting the experimental settings of RIC and the other competing methods. Following that, we introduce the evaluation criteria. We then report the results on the classification accuracy, confidentiality and recovery performance. Additionally, we investigate the generalization of our model to different datasets, and analyze the impact of critical elements in the RIC.

 \subsection{Experimental Settings}\label{section:exset}
	\subsubsection{Datasets} We mainly conduct experiments on two datasets: Street View House Numbers (SVHN) \cite{netzer2011reading} and VGGFace2\cite{cao2018vggface2}, which contain privacy-sensitive house numbers and human faces, respectively. Specifically, SVHN dataset consists of 90,000 images of size $32\times32$, while VGGFace2 dataset has a training set of 3 million images of 8,631 identities. We use aligned and cropped faces from the VGGFace2 dataset, which are resized to $160\times160$. Additionally, to evaluate the generalization of our approach to different datasets, we further perform tests on other datasets, including the small-scale Cifar-10\cite{krizhevsky2009learning} (10 classes with size $32\times32$ ) and the large-scale ImageNet\cite{deng2009imagenet} (1000 classes). Images from ImageNet are aligned and cropped to a size of $256\times256$. To enable a comparison with existing PPML methods, we also conduct experiments on the MNIST dataset.        

	\subsubsection{Classifiers}
 In our experiments, we choose publicly available classifiers that are both widely used and high-performing as $\mathcal{C}$ in RICs. For SVHN, we use a CNN with 7 convolutional layers and 1 fully connected layer \cite{playground}. For VGGFace2, we adopt the FaceNet\cite{schroff2015facenet} with a linear classification layer. For Cifar-10 and ImageNet, the adopted classifiers are ResNet-18\cite{he2016deep} and ResNet-50\cite{he2016deep}, respectively. For MNIST, the classifier is a simple CNN with 5 convolutional layers and 1 fully connected layer\cite{an2020ensemble}. These classifiers all achieve good classification performance, i.e., 96.00/90.20/97.09/77.80/99.75\% accuracy for the 10/8,631/10/1,000/10 classes image classification task on SVHN/VGGFace2/Cifar-10/ImageNet/MNIST test sets, respectively.

	\subsubsection{Implementation Details} 
	During the training process of our RIC, we employ the Adam optimizer to minimize the training loss with a learning rate of $1e-4$. We set $\lambda=0.8$ in (\ref{eq:en1}) and (\ref{eq:decoder1}), $\beta=10$ in (\ref{eq:loss2}), $g_1=5$ in (\ref{eq:clsnadv}), and $g_2=5$ in (\ref{eq:loss1}).

	\subsection{Comparative Methods}
 To assess the effectiveness of our proposed RIC system, we conduct a comparative evaluation against two categories of methods: SOTA PPML approaches that focus on classification accuracy in the encrypted domain and traditional steganography-like approaches. The former type includes MiniONN\cite{liu2017oblivious}, ARDEN\cite{wang2018not}, TAPAS\cite{sanyal2018tapas}, Cloak\cite{mireshghallah2021not} and InstaHide\cite{huang2020instahide}. The latter type encompasses LSB\cite{lsb1998}, IICNet\cite{cheng2021iicnet}, HiNet\cite{jing2021hinet}, and DDH\cite{zhang2020udh}. These steganography-like techniques were initially devised to obfuscate images within a host image, commonly referred to as a container image, thereby producing a privacy-protected image. In these methods, natural images are used as container images during the training process. To facilitate fair comparisons, we have re-implemented them by replacing the container images with NAEs, which aligns with the training process of RIC. For LSB, we choose a commonly used 4-bit scheme for encoding images.

	\subsection{Criteria}
	To evaluate the recovery performance, we use LPIPS\cite{zhang2018unreasonable}, PSNR, and SSIM metrics. 
 In terms of the classification performance, we calculate the classification accuracy on the encrypted images, denoted as $\mathop{ACC_e}$, either using the plaintext-domain classifier $\mathcal{C}$, as in our RIC, or dedicated classifier in other PPML schemes. Let also $ACC_p$ represents the classification accuracy on plaintext images via a plaintext-domain classifier. We then define the Accuracy Drop Proportion ($\mathop{ADP}$) metric to measure the classification accuracy drop between plaintext and encrypted domains:      

 \begin{equation}\label{eq:drop}
		\begin{aligned}		
			\mathop{ADP} =\frac{ACC_p- ACC_e}{ACC_p}\times 100\%.
		\end{aligned}
	\end{equation} Clearly, smaller $\mathop{ADP}$ is highly desirable, as the encrypted-domain classification accuracy gets close to that of the plaintext domain.

% Table generated by Excel2LaTeX from sheet '工作表2'
\begin{table}[t!]
  \centering
  \caption{The comparisons of the generalization ability between RIC and IICNet\cite{cheng2021iicnet} models trained on SVHN and VGGFace2 (RIC-SVHN,RIC-VGGFace2,IICNet-SVHN,IICNet-VGGFace2). The items marked with $*$ indicate that the training and test data are from the same dataset.}
  \scalebox{0.7}{
    \begin{tabular}{c|c|c|c|c|c|c|c|c|c|c|c|c}
    \hline\hline
     Datasets & \multicolumn{3}{c|}{Cifar-10} & \multicolumn{3}{c|}{ImageNet} & \multicolumn{3}{c|}{SVHN} & \multicolumn{3}{c}{VGGFace2}\Tstrut\\
     \hline
Models          & PSNR$\uparrow$  & SSIM$\uparrow$  & LPIPS$\downarrow$ & PSNR$\uparrow$  & SSIM$\uparrow$  & LPIPS$\downarrow$ & PSNR$\uparrow$  & SSIM$\uparrow$  & LPIPS$\downarrow$ & PSNR$\uparrow$  & SSIM$\uparrow$  & LPIPS$\downarrow$ \Tstrut\\
    \hline
    RIC-SVHN & \textbf{40.03} & \textbf{0.9963} & \textbf{0.0004} & \textbf{41.43} & \textbf{0.9839} & \textbf{0.0029} & \textbf{51.46*} & \textbf{0.9995*} & \textbf{0.0001*} & 47.76 & 0.9971 & 0.0020 \Tstrut\\
    \hline
     RIC-VGGFace2  & 36.94 & 0.9920 & 0.0011 & 38.96 & 0.9740 & 0.0079 & 45.35 & 0.9981 & 0.0007 & \textbf{48.04*} & \textbf{0.9972*} & \textbf{0.0015*}\Tstrut\\
    \hline
    IICNet-SVHN & 26.33& 0.9207 & 0.0513 & 29.65 & 0.9109 & 0.1850 & 38.59* & 0.9908* & 0.5508* & 34.09 & 0.9657 & 0.0828  \Tstrut\\
    \hline
    IICNet-VGGFace2 & 22.43 & 0.8309 & 0.1305 & 26.72 & 0.8300  & 0.3291 & 29.80  & 0.9462 & 0.0258 & 34.06* & 0.9535* & 0.0853* \Tstrut\\
    \hline\hline
    \end{tabular}%
    }
  \label{tab:gen}%
\end{table}%

 \subsection{Accuracy, Confidentiality, and Generalization Performance}\label{sec:acc}
We now present experimental results of comparative methods in terms of our three design goals, i.e. classification accuracy, confidentiality, and generalizability.

 1) \textbf{Accuracy}: The comparison of classification accuracies is summarized in Table \ref{tab:ppmlacc}. It can be observed that our method maintains a 0\% accuracy drop across various test datasets. This is because our method directly utilizes classifiers trained in the plaintext domain for classifying ciphertexts, without requiring dedicated classifiers that may have components detrimental to accuracy. Additionally, the combination of NAEs and our encoder enhances the adversarial capability of the resulting ciphertexts, enabling ciphertexts to effectively deceive the classifier. In contrast, other comparative approaches suffer from varying degrees of accuracy degradation, and descents are particularly severe on large-scale datasets. For example, in ImageNet, the accuracy of Cloak and TAPAS drops from 77.80\% to only $23.19\%$ and $7.84\%$, respectively. Although InstaHide experiences a mere 6\% decline, it's worth noting that, their generated ciphertexts cannot guarantee a sufficiently secure privacy protection effect, as acknowledged by the method itself \cite{huang2020instahide}. Noting that both MiniONN and ARDEN lack accuracy performance on ImageNet, denoted by "-" in Table \ref{tab:ppmlacc}, as the principle of MiniONN renders this approach impractically time-consuming for communication on ImageNet, and the ARDEN method lacks many crucial implementation details for ImageNet.
 
2) \textbf{Confidentiality}: In terms of privacy-preserving performance, thanks to the noisy appearance of NAEs, ciphertexts generated by competing methods all resemble noise-like images, as depicted in Fig.\ref{fig:samples}(a). Thus, we illustrate ciphertexts of only one image randomly selected from the VGGFace2 test dataset as an example here. Clearly, these ciphertexts composed primarily of noisy textures make them challenging for human observers to discern any meaningful semantic information.

When it comes to recoverability, our RIC exhibits superior performance than comparative techniques. As can be seen in Fig.\ref{fig:samples}(b), images recovered by our RIC maintain much better texture details than IICNet, HiNet and DDH, such as teeth and lip edges of faces at the first row (see red blocks), and skin texture at the second row (see purple blocks). Meanwhile, our RIC preserves more satisfactory color fidelity than IICNet, DDH and LSB. For instance, the skin tones restored by IICNet do not align with the plaintext. DDH and LSB introduce a plethora of undesired colorful artifacts across the entire visage.

In addition, we provide quantitative results for image recovery performance on the SVHN and VGGFace2 datasets in Table \ref{tab:total1}. For fair comparisons, the container images for concealing the plaintext images in compared methods (LSB, IICNet, HiNet and DDH) are substituted with the same NAE employed in our RIC during both the training and testing stages. It is demonstrated that our approach performs much better compared to other approaches, such as achieving approximately 13/14 dB higher PSNR on SVHN/VGGFace2 compared to the second-best method, i.e. IICNet.

3) \textbf{Generalization}: We now conduct experiments to evaluate RIC's generalization ability concerning the encryption and decryption functionalities. This refers to the ability of a well-trained RIC to encrypt and decrypt a test dataset different from the training dataset (\textit{cross-dataset}) \cite{nadimpalli2022improving}. Considering that IICNet demonstrates superior recovery performance among steganography-like methods (see Table \ref{tab:total1}), we here only compare RIC and IICNet in terms of their generalizability. There are four comparative models in total: RIC-SVHN and IICNet-SVHN trained on the SVHN dataset, as well as RIC-VGGFace2 and IICNet-VGGFace2 trained on the VGGFace2 dataset. These four models are then tested on four different datasets, each consisting of 500 randomly selected images from CIFAR-10, ImageNet, SVHN and VGGFace2 datasets, respectively. Here, the latter two datasets correspond to the in-dataset scenarios, serving as performance baselines.

The results are compiled in Table \ref{tab:gen}. Specifically, in terms of recoverability, our proposed RIC outperforms IICNet in all cross-dataset cases. For example, on Cifar-10 and ImageNet test datasets, RIC-SVHN and RIC-VGGFace2 both achieve over 9dB PSNR gains than IICNet-SVHN and IICNet-VGGFace2. Particularly, RIC-SVHN obtains 17.6dB higher than IICNet-VGGFace2 on Cifar-10. Such a superior recoverability allows us to apply the RIC to images from various datasets with only a one-time training cost. It should be noted that our RIC achieves even much better recovery performance in the cross-dataset scenario than IICNet in the in-dataset scenario. For instance, on VGGFace2, RIC-SVHN obtains a PSNR of 47.76dB, while IICNet-VGGFace2 only reaches a much lower 34.06dB. Besides the above quantitative results, we also present the visualizations of recovered images by comparative models in Fig.\ref{fig:generalization}(a). It can be observed that RIC-recovered images resulting from totally different datasets are all nearly indistinguishable from their corresponding plaintexts. In contrast, as can be seen from the last two rows of Fig.\ref{fig:generalization}(a), recovered images produced by IICNet models exhibit severe artifacts such as noticeable color distortion, noise, and blurriness.

Regarding the privacy-preserving capability, as illustrated in Fig.\ref{fig:generalization}(b), all the compared models have demonstrated commendable performance. By a close inspection, we notice that the RIC-VGGFace2 has the capability to conceal a greater amount of visual information in its ciphertext compared to RIC-SVHN. It is then recommended to use a RIC trained on a relatively large-scale dataset, which could benefit the generalization tasks. Also, in these cases, the $ADP$ values are consistently 0\%, due to the utilization of our NAEs for hiding the plaintext during both their training and testing stages.

\begin{center}
		\begin{figure}[t]			\centering{\includegraphics[width=1\textwidth]{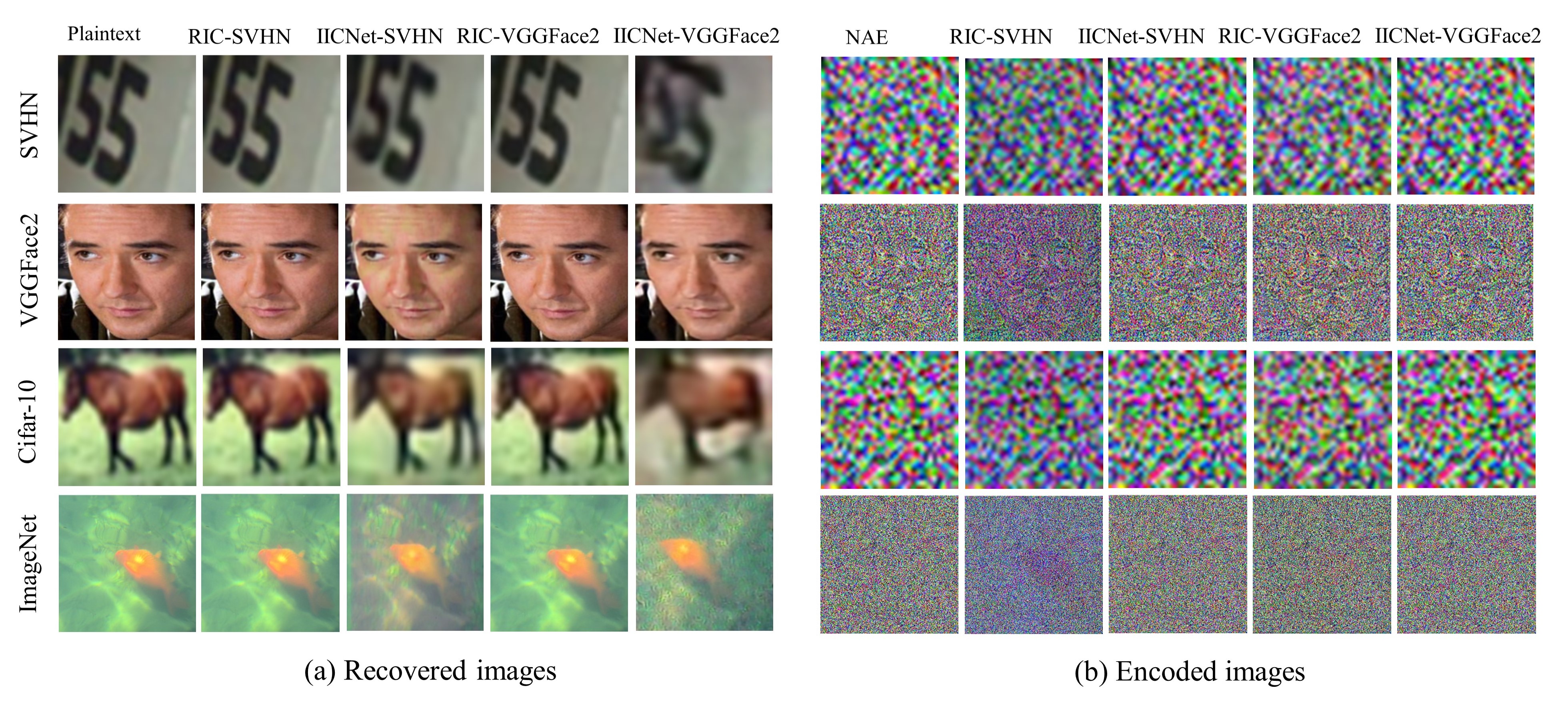}}
			\caption{The visualization of models' generalization on different datasets.}
			\label{fig:generalization}
		\end{figure}
	\end{center}

\begin{table}[t]
  \centering
  \caption{The influence of the NAE and SRL in the RIC. Compared models are trained and tested on the SVHN dataset.}
  \scalebox{0.85}{
    \begin{tabular}{c|c|c|c|c}
    \hline
    \hline
      Metrics&Accuracy& \multicolumn{3}{c}{Recoverability}  \Tstrut\\
     \hline
 Modules & $ADP\downarrow$&PSNR$\uparrow$ & SSIM$\uparrow$  & LPIPS$\downarrow$    \Tstrut\\
    \hline
    \textbf{NAE w/ SRL (Our RIC)} & \textbf{0\%}  & \textbf{51.46} & \textbf{0.9995}  & \textbf{0.0001}\Tstrut\\
 \hline
    RNI w/ SRL & 4.5\% & 31.81 & 0.9671   & 0.0200 \Tstrut\\ 
     \hline
    NAE w/o SRL & \textbf{0\%}    & 50.45 & 0.9993  & \textbf{0.0001} \Tstrut\\
    \hline
    \hline
    \end{tabular}%
    }
  \label{tab:abation2}%
\end{table}%

\subsection{The Influence of the NAE, SRL and $\lambda$}\label{sec:abation}
To study the influence of NAEs and SRL, we train our system separately by controlling the utilization of NAEs and SRL, while keeping other settings unchanged. These compared modules are dubbed as 1) NAE w/ SRL: the intact RIC; 2) RNI w/ SRL: eliminating the adversarial capability of NAE by replacing $\mathbf{n_{adv}}$ in (\ref{eq:en1}) and (\ref{eq:decoder1}) with the associated RNI, i.e. $\mathbf{n}$; and 3) NAE w/o SRL: dismissing the SRL by replacing (\ref{eq:en1}) with $\mathbf{x^\prime}=\mathbf{u_{en}}$ and substituting $\mathbf{u_{en}^\prime}=(\mathbf{x^{\prime}_{q}}||\mathbf{n_{adv}})$ in (\ref{eq:decoder1}). 

The experimental results regarding the classification accuracy and recoverability are reported in Table \ref{tab:abation2}, while privacy-preserving performance is illustrated in Fig.\ref{fig:ablation1}. By comparing the third and fourth rows (i.e. NAE w/ SRL (Our RIC) and RNI w/ SRL) in Table \ref{tab:abation2}, we observe that replacing the NAE with the corresponding RNI leads to a significant decrease in recovery PSNR by $19+$dB and an accuracy drop of $4.5\%$. This finding highlights the importance of the adversarial capability of NAEs in enabling the network to better classify the ciphertext as the desired label, which in turn improves the network's ability to reconstruct images. When comparing NAE w/ SRL with NAE w/o SRL, we find that the incorporation of SRL enhances image recoverability, resulting in a 1dB gain in PSNR. As for the privacy-preserving performance, Fig.\ref{fig:ablation1} visually illustrates that both the absence of the NAE and SRL compromise the system's ability to conceal significant visual information.

With regard to the parameter $\lambda$, we train different RICs on the SVHN dataset by varying $\lambda$ in the range of $(0,1)$, and then evaluate their performance in SVHN (in-dataset) or Cifar-10 (cross-dataset) test datasets. The experimental results are presented in Fig.\ref{fig:lambda}. As can be seen from Fig.\ref{fig:lambda}(a), varying $\lambda$s in the range $[0.1,0.8]$ lead to rather stable accuracy for both in-dataset and cross-dataset scenarios, while the accuracy slightly drops for cross-dataset scenario when $\lambda$ further increases. In addition, from Fig.\ref{fig:lambda}(b), the recoverability tends to increase with respect to the increasing $\lambda$. This could be attributed to the fact that a bigger $\lambda$ implies a larger proportion of NAEs incorporated into ciphertexts according to (\ref{eq:en1}). This, in turn, enhances the adversarial capability of ciphertexts, allowing the RIC to focus more on the image recovery task during training. Regarding the privacy-preserving performance, it fluctuates with very small amplitudes when changing $\lambda$. Therefore, we choose $\lambda=0.8$ in all our experiments, striking a good balance among the above evaluation metrics.

\begin{center}
		\begin{figure}[tbp]
			\centering{\includegraphics[width=0.8\textwidth]{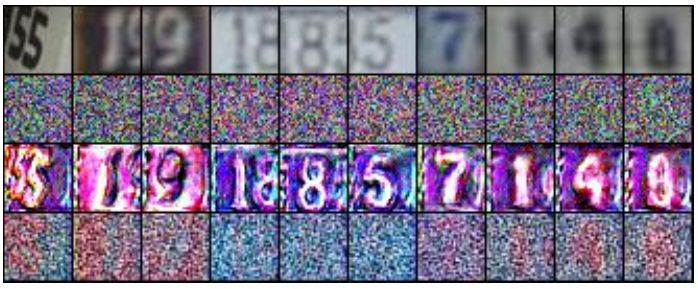}}
			\caption{The influence of the NAE and SRL on the privacy-preserving performance. The first to the fourth rows correspond to plaintext images, encrypted images by NAE w/ SRL (our RIC), RNI w/ SRL and NAE w/o SRL, respectively. Compared models are trained and tested on the SVHN dataset.}
			\label{fig:ablation1}
		\end{figure}
	\end{center}

\section{Security Analysis and Evaluation}\label{sec:sec6}

In this section, we would like to assess the security of our proposed system against various types of attacks, including the brute-force attack, cloud attack, and known-plaintext attack.

	\begin{center}
		\begin{figure}[t]			\centering{\includegraphics[width=1\textwidth]{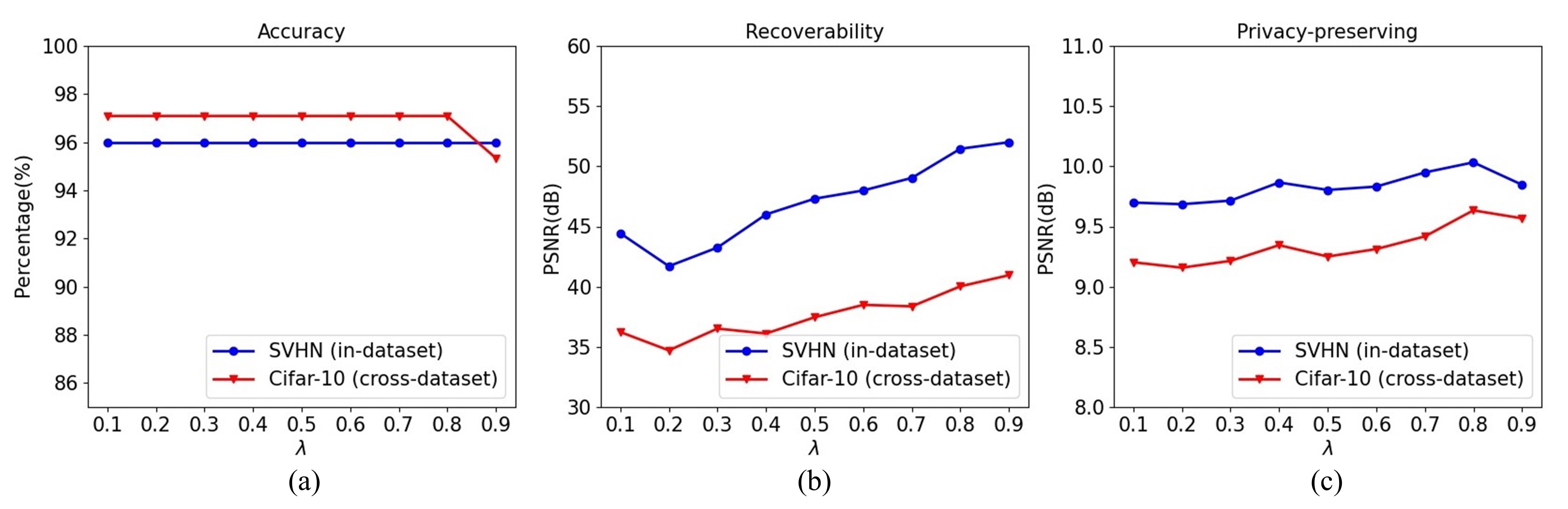}}
			\caption{The influence of $\lambda$ on classification accuracy, recoverability, and privacy-preserving performance in both in-dataset and cross-dataset scenarios.}
			\label{fig:lambda}
		\end{figure}
	\end{center}
\subsection{Brute-force Attack}
An adversary has obtained a ciphertext $\mathbf{x_q^\prime}$, along with the knowledge of the adopted adversarial attack method $\Phi$ and the decoder $\mathcal{D}$, without the secret key $k$. One of the simplest and most straightforward approaches for the adversary is to perform a brute-force attack by exhaustively traversing the key space. To evaluate the security of RIC under such a brute-force attack, we conduct a theoretical analysis, augmented with empirical justifications. Our objective is to demonstrate that it is highly unlikely to recover a plaintext image similar to the original plaintext from a given ciphertext through a brute-force attack. To this end, we have the following Theorem:

\begin{theorem}\label{th1} For an 8-bit d-dimensional image $\mathbf{x_q^\prime}\in\mathbb{R}^d$ encrypted by our proposed method with a secret key $k$, there exists a positive constant 
$\alpha>0$ such that the following inequality holds: 
	\begin{equation}\label{eq:prob} 		\prob{\mathbb{E}_{p(\mathbf{x^\prime_q})}[\mathop{PSNR}(\tilde{\mathbf{x}}^\prime,\tilde{\mathbf{x}})]\geq\frac{\alpha}{m}}\leq\left(\frac{2m+1}{256}\right)^d.
	\end{equation} where $m = ||\tilde{\mathbf{x}}-\tilde{\mathbf{x}}^\prime||_\infty$, $\tilde{\mathbf{x}}$ and $\tilde{\mathbf{x}}^\prime$ represent the images decrypted from $\mathbf{x_q^\prime}$ using the correct key $k$ and a randomly guessed key $k^\prime$, respectively. Also, the constant $\alpha>0$ is specific to the dataset $p(\mathbf{x_q^\prime})$ to which this theorem is applied. 
\end{theorem}

\begin{proof}
   Please refer to Appendix \ref{sec:lemmap} for detailed verification.
\end{proof}

By noticing Theorem \ref{th1}, we can deduce that as the value of $m$ decreases, the probability term $\left(\frac{2m+1}{256}\right)^d$ also decreases, while the term $\frac{\alpha}{m}$ increases. This suggests that the PSNR similarity between $\tilde{\mathbf{x}}$ and $\tilde{\mathbf{x}}^\prime$ is inversely proportional to the probability of achieving this level of similarity. Also, to demonstrate that the occurrence probability of the image reconstructed by $k^\prime$ resembling the image recovered by $k$, i.e., the success of the brute-force attack, is extremely low, we analyze the maximum PSNR that can be achieved for a specific dataset and its probability when $m=1$. 

For the VGGFace2 dataset, which contains plaintext RGB images with $d=160\times160\times3\approx7.7\times10^5$, considering $\frac{2m+1}{256}\approx0.01$ when $m=1$, we have the following inequality:

	\begin{equation} \label{eq:vggbfa}\prob{\mathbb{E}_{p(\mathbf{x^\prime_q})}[\mathop{PSNR}(\tilde{\mathbf{x}}^\prime,\tilde{\mathbf{x}})]\geq5.64}\leq10^{-1.54\times10^6}.
	\end{equation} 
Here $\alpha=5.64$ is obtained by using the definition formula of PSNR, combined with the Natural Evolution Strategy (NES) estimation method \cite{llyas2018blackattack} (see Appendix \ref{sec:lemmap} for more details). The above inequality suggests that the upper bound of the probability that the PSNR between the image recovered by the correct key $k$ and a brute-force guessed key $k^\prime$ is larger than 5.64 is nearly 0. This means that a successful brute-force attack is highly unlikely. The estimation of $\alpha$ is also consistent with the experimental result shown in Fig.\ref{fig:seedsen}(a), from which we can observe that wrong keys yield poor performance, with mean/minimum/maximum PSNR values of 4.84/4.38/5.30 dB, which are all smaller than the theoretical value of 5.64 dB. Meanwhile, only the correct key achieves satisfactory recovery performance, with a PSNR of 47+ dB. The very narrow peaks and basin structure in Fig.\ref{fig:seedsen} also indicate satisfactory key sensitivity of our RIC. Furthermore, in the second line of Fig.\ref{fig:threatattack}, we provide several examples of images recovered by incorrect keys, which convey no semantic information of the original images.

\begin{center}
		\begin{figure}[t]			\centering{\includegraphics[width=1\textwidth]{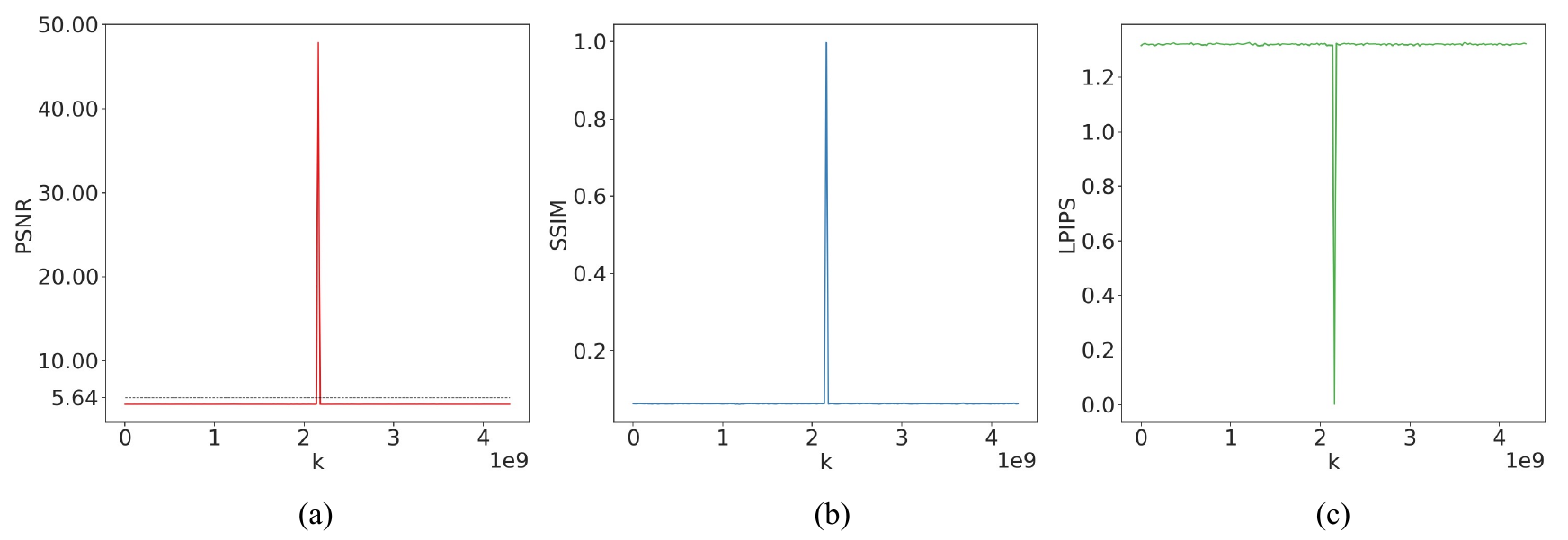}}
			\caption{Average PSNR, SSIM, and LPIPS between 100 pairs of plaintext images and recovered images under a brute-force attack. The plaintext images are randomly sampled from the VGGFace2 test dataset, and recovered images are decoded by using 200 different seeds uniformly sampled from the range $[0,2^{32}-1]$.}
			\label{fig:seedsen}
		\end{figure}
	\end{center}

\subsection{Cloud Attack}\label{sec:cattack}
We now consider an honest but curious cloud server that possesses the target classifier $\mathcal{C}$ and a set of historical ciphertexts. Harnessing these existing assets, the cloud server aims to recover plaintexts from the ciphertexts. We design a novel generative approach that leverages a discriminator to assist a generator in producing images as the recovered plaintexts. This discriminator can distinguish real images from synthesized counterparts. The images produced by this generator are expected to fulfill two criteria: 1) they need to be classified by $\mathcal{C}$ as belonging to the same class as the corresponding ciphertext and 2) they should be identified as real samples by the discriminator but not as synthesized samples. Additionally, we introduce a regularization term in the training loss of this generator to promote diversity in the generated content. More details on the loss function and training procedure are presented in Appendix \ref{app2}. Next, we examine the reconstructed images from the VGGFace2 dataset by the generator trained by this method.

We randomly sample 500 ciphertexts encrypted by the RIC for the testing purposes. The corresponding face images recovered by the cloud attack are displayed in the third row of Fig.\ref{fig:threatattack}. It is evident from the visualizations that the well-trained generator struggles to recover recognizable human face images from the ciphertexts. This highlights the difficulty in achieving meaningful recovery through the cloud attack.

  \begin{center}
		\begin{figure}[t]	\centering{\includegraphics[width=1\textwidth]{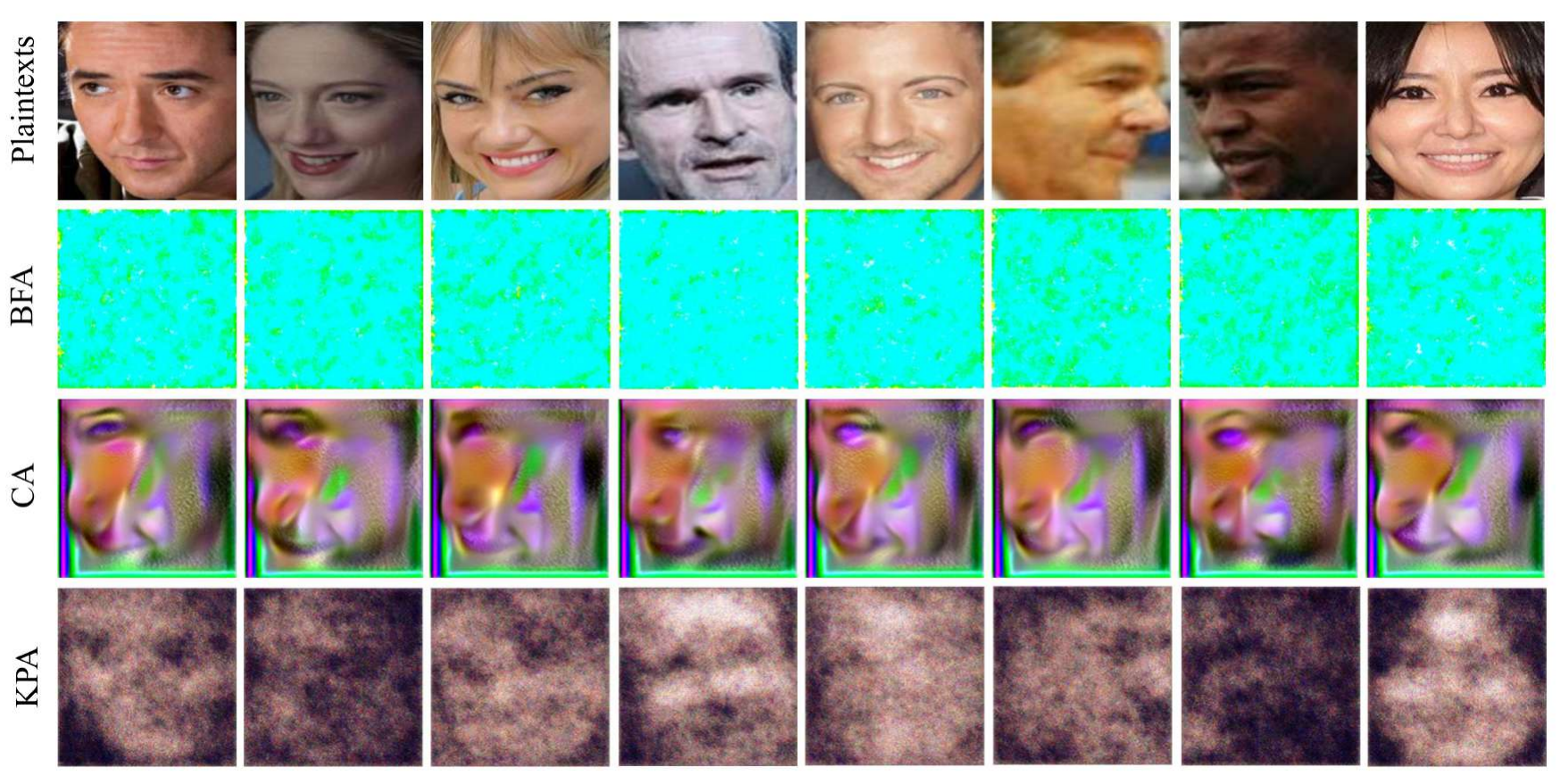}}
			\caption{Visualization of plaintexts from the VGGFace2 dataset and images recovered by a brute-force attack (BFA), cloud attack (CA) and known-plaintext attack (KPA).}
			\label{fig:threatattack}
		\end{figure}
	\end{center}
\subsection{Known-plaintext Attack} 
A more powerful attack is the known-plaintext attack, where adversaries have access to pairs of past ciphertexts and corresponding plaintext images. This scenario can occur in our system when adversaries are previously authorized users and have obtained some images that have been successfully recovered with the correct key. In such a scenario, adversaries can form a set of $N$ pairs of ciphertexts and corresponding correctly recovered plaintexts images as $\mathbb{A}=\left\{(\mathbf{x}^\prime_{i},\Tilde{\mathbf{x}_i})\mid 0\leq i\leq N, i, N \in\mathbb{Z}\right\}$. To recover the underlying mapping between the ciphertext and plaintext images, our approach is to train a network $\mathcal{M}$ by optimizing the following objective:

\begin{equation}\label{eq:reconatt} 
\min_{\theta_\mathcal{M}}\frac{1}{N}\sum_{i=1}^{N}||\mathcal{M}(\mathbf{x}^\prime_{i})-\Tilde{\mathbf{x}_i}||_2.
	\end{equation}Here, $\theta_\mathcal{M}$ represents the parameters of $\mathcal{M}$. In our experiment, we implement $\mathcal{M}$ using the U-Net\cite{ronneberger2015u} architecture, due to its excellent performance in low-shot learning scenarios, since adversaries in this attack context typically have access to only a small number of training samples. We train a U-Net model on the VGGFace2 dataset with $N=500$. The reconstructed images of randomly sampled ciphertexts from the test dataset are displayed in the fourth row of Fig.\ref{fig:threatattack}. It is shown that recognizable faces can hardly be reconstructed from the ciphertexts generated by our RIC.

\section{Conclusion}\label{sec:sec7}
In this paper, we have introduced a recoverable privacy-preserving classification system that enables a classifier pre-trained in the plaintext domain to work effectively in the ciphertext domain. This relieves the burden of re-training dedicated ciphertext-domain classifiers, which often lead to severely compromised classification accuracy. In addition, the produced ciphertext images can be faithfully decoded into plaintext images with negligible distortions, e.g., achieving 51+/48+dB PSNR on SVHN/VGGFace2 datasets. Extensive experiments have been provided to validate the superior performance of our proposed scheme in terms of ciphertext-domain classification accuracy, confidentiality, and generalizability.

\bibliographystyle{ACM-Reference-Format}
\bibliography{reference}

%\appendix[Proof of the Key Sensitivity]
\appendix
\section{Proof of Theorem \ref{th1}}\label{sec:lemmap}
%\begin{theorem}

% \begin{lemma}
% \end{lemma}

We first analyze the lower bound of the expectation of PSNR, denoted as $\frac{\alpha}{m}$ in (\ref{eq:prob}), by starting from the PSNR definition. It is well known that PSNR between images $\tilde{\mathbf{x}}^\prime$ and $\tilde{\mathbf{x}}$ can be expressed as: 
  	\begin{equation}\label{eq:psnr}
		\mathop{PSNR}(\tilde{\mathbf{x}}^\prime,\tilde{\mathbf{x}})=20\cdot\log_{10}\frac{\mathop{MAX_{I}}}{\sqrt{MSE(\tilde{\mathbf{x}}^\prime,\tilde{\mathbf{x}})}}.       
	\end{equation} 
Here, $\mathop{MAX_{I}}$ represents the maximum possible pixel value of images, which is 255 in our case. Also, the Mean Square Error (MSE) is calculated by:  
	  	\begin{equation}\label{eq:mse}	
			\mathop{MSE}(\tilde{\mathbf{x}}^\prime,\tilde{\mathbf{x}})=\frac{1}{d}\sum_{i=1}^{d}(s_i^\prime-s_i)^2,
	\end{equation}
 where $s_i$ and $s_i^\prime$ are the $i$th pixel values of $\tilde{\mathbf{x}}$ and $\tilde{\mathbf{x}}^\prime$, respectively. By scrutinizing the numerator and denominator of (\ref{eq:psnr}), we discern that $\mathop{PSNR}(\tilde{\mathbf{x}}^\prime,\tilde{\mathbf{x}})$ is predominantly determined by $s_i-s_i^\prime$. Consequently, to obtain the value of $\frac{\alpha}{m}$, we embark on an initial analysis of $s_i$ and $s_i^\prime$. To facilitate subsequent analytical discourse, we define a new function $\mathcal{H}_i(\mathbf{r}):\mathbf{r}\in\mathbb{R}^d\to s_i\in\mathbb{R}$ to represent the mapping from an RNI $\mathbf{r}$ to the $i$th pixel value of its associated recovered image. Without loss of generality, we have the following assumption:
	 
	 \textbf{Assumption 1}: The function $\mathcal{H}_i(\mathbf{r})$ is differentiable, allowing us to estimate its gradient with respect to the input $\mathbf{r}$. 
	 
	 According to the gradient theorem\cite{zill2009multivariable}, if $\mathcal{H}_i(\mathbf{r})$ is a differentiable mapping from $\mathbb{R}^d$ to $\mathbb{R}$, and $\gamma$ is a continuous curve in $\mathbb{R}^d$ that starts at a point $\mathbf{n}$ and ends at another point $\mathbf{n}^\prime$, then the integral of the gradient of $\mathcal{H}_i(\mathbf{r})$ along the curve $\gamma$ is equal to the difference in the values of $\mathcal{H}_i(\mathbf{r})$ evaluated at $\mathbf{n^\prime}$ and $\mathbf{n}$. Let $\gamma$ be parameterized by a function of $t$ as $\mathbf{r}(t)=t\cdot\mathbf{n}^\prime+(1-t)\cdot\mathbf{n}$, where $0
  \leq t\leq1$, and $\mathbf{n}^\prime\in\mathcal{B}(\mathbf{n},\epsilon)$. Here, RNIs $\mathbf{n}$ and $\mathbf{n}^\prime$ are generated by $\mathop{PRNG}(k,U)$ and $\mathop{PRNG}(k^\prime,U)$ respectively. Also, $\mathcal{B}$ represents a high-dimensional $l_\infty$ ball centered at $\mathbf{n}$, and $\epsilon\in\{1,2,...,255\}$ is the radius of the ball. Therefore, the difference between pixel values of the same $i$th dimensionality is no more than $\epsilon$, namely $|n_i^\prime-n_i|\leq\epsilon$. Consequently, by using the gradient theorem, we can obtain an upper bound for $s_i^\prime-s_i$, i.e. $\mathcal{H}_i(\mathbf{n^\prime})-
		\mathcal{H}_i(\mathbf{n})$, through the following derivation:

	\begin{equation}\label{eq:resbound}
		\begin{aligned}	
		\mathcal{H}_i(\mathbf{n^\prime})-
		\mathcal{H}_i(\mathbf{n})
		&= \int_{\gamma}\nabla_{\mathbf{r}}\mathcal{H}_i(\mathbf{r})\cdot d\mathbf{r}\\
			&= \int_{0}^{1}\nabla_{\mathbf{r}}\mathcal{H}_i(\mathbf{r})\cdot \nabla_t\mathbf{r}(t)\cdot dt\\
			&=\int_{0}^{1}||\nabla_{\mathbf{r}}\mathcal{H}_i(\mathbf{r})||_2 ||\nabla_t\mathbf{r}(t)||_2\cos(\nabla_{\mathbf{r}}\mathcal{H}_i(\mathbf{r}),\nabla_t\mathbf{r}(t))dt\\
			&\leq\int_{0}^{1}||\nabla_{\mathbf{r}}\mathcal{H}_i(\mathbf{r})||_2 ||\nabla_t\mathbf{r}(t)||_2dt\\
			&\leq\int_{0}^{1}||\nabla_{\mathbf{r}}\mathcal{H}_i(\mathbf{r})||_2 ||\mathbf{n}^\prime-\mathbf{n}||_2dt\\
			&\leq\int_{0}^{1}||\nabla_{\mathbf{r}}\mathcal{H}_i(\mathbf{r})||_2\cdot\sqrt{\sum_{i=1}^{d}(n_i^\prime-n_i)^2}\cdot dt.\\
   &\leq\epsilon\cdot\sqrt{d}\cdot(||\nabla_{\mathbf{n}^\prime}\mathcal{H}_i(\mathbf{n}^\prime)||_2-||\nabla_{\mathbf{n}}\mathcal{H}_i(\mathbf{n})||_2).
		\end{aligned}
	\end{equation}

%第二个小于等于是利用求导法则计算的，第三个小于等于是根据2范数的求法定义来的

  Afterwards, by combining (\ref{eq:psnr}), (\ref{eq:mse}) and (\ref{eq:resbound}), we obtain the following inequality:
	
 \begin{equation}\label{eq:psnr2}
		%\begin{align*}		
			\mathop{PSNR}(\tilde{\mathbf{x}}^\prime,\tilde{\mathbf{x}})\geq20\cdot\log_{10}\frac{255}{\epsilon\cdot\sqrt{\sum_{i=1}^{d}(||\nabla_{\mathbf{n}^\prime}\mathcal{H}_i(\mathbf{n}^\prime)||_2-||\nabla_{\mathbf{n}}\mathcal{H}_i(\mathbf{n})||_2)}}.
	\end{equation} 
 For simplicity, we represent the right-hand side term as $\alpha/\epsilon$. To further determine the value of $\alpha/\epsilon$, we utilize the NES method\cite{llyas2018blackattack} to estimate the gradient of $\mathcal{H}$ as follows:
  
	\begin{equation}\label{eq:nes}
		\begin{aligned}		
			\nabla_{\mathbf{r}}\mathcal{H}_i(\mathbf{r})\approx\frac{1}{\sigma z}\sum_{j=1}^{z}\mathbf{v}_j\mathcal{H}_i(\mathbf{r}+\sigma\mathbf{v}_j).
		\end{aligned}
	\end{equation}
 Here, the hyper-parameter $z$ represents the number of the vector $\mathbf{v}_j$ that are randomly sampled from a standard Gaussian distribution, and $\sigma$ is a small positive constant to control the estimation accuracy. Based on (\ref{eq:nes}), we can estimate the value of $\alpha$ for a specific dataset by sampling ciphertexts of this dataset and then generating their recovered images using $\mathbf{n}$ or $\mathbf{n}^\prime$. To this end, we are able to obtain the estimated expression of the left side in (\ref{eq:prob}). 

 Subsequently, to substantiate the validity of (\ref{eq:prob}), we examine its right-hand side, which represents the upper bound of the probability related to $\mathbb{E}_{p(\mathbf{x^\prime_q})}[\mathop{PSNR}(\tilde{\mathbf{x}}^\prime,\tilde{\mathbf{x}})]$. We extrapolate this probability by leveraging the likelihood of $\epsilon$'s values in (\ref{eq:psnr2}). Specifically, since $\epsilon$ is the $l_\infty$ norm of $\mathbf{n}^\prime-\mathbf{n}$, and the pixel values of $\mathbf{n}^\prime$ and $\mathbf{n}$ are within the set $\{1,2,...,255\}$, the probability that $\epsilon\leq m$, $m\in\{1,2,...,255\}$, can be inferred as follows:
	\begin{equation}\label{eq:epsilonp}
		\begin{aligned}			\prob{\epsilon\leq m}\leq\left(\frac{2m+1}{256}\right)^d.
		\end{aligned}
	\end{equation}
Eventually, according to (\ref{eq:psnr2}) and (\ref{eq:epsilonp}), we can finally derive: 

 \begin{equation}\label{eq:psnr5}
			\prob{\mathbb{E}_{p(\mathbf{x^\prime_q})}[\mathop{PSNR}(\tilde{\mathbf{x}}^\prime,\tilde{\mathbf{x}})]\geq\frac{\alpha}{m}}\leq\left(\frac{2m+1}{256}\right)^d.
	\end{equation}

\section{More details on the cloud attack}\label{app2}
In the cloud attack, the cloud server possesses the target classifier $\mathcal{C}$ and a set of ciphertexts $\mathbf{x^\prime_q}$, with predicted class label being $\hat{y}=\mathop{\arg\max}_{i}\mathcal{C}(\mathbf{x^\prime_q})[i]$. We design a new generative method that leverages a discerning discriminator $\mathbf{D}$ that can distinguish genuine images from synthesized counterparts. This discriminator guides a generator $\mathcal{G}$ in crafting an image $\hat{\mathbf{x}}=\mathcal{G}(\mathbf{E}(\mathbf{x^\prime_q}))$ that resembles a genuine one and be recognized as the same class label of $\mathbf{x^\prime_q}$ by $\mathcal{C}$. Here, $\mathbf{E}$ represents an encoder that produces a latent feature $\mathbf{w}$ from the ciphertext $\mathbf{x^\prime_q}$. The parameters $\theta_\mathcal{G}$ of the generator $\mathcal{G}$ and $\theta_\mathcal{E}$ of the encoder $\mathbf{E}$ are learned by optimizing the following objective: 
\begin{equation}\label{eq:cloudatt} 
\min_{\theta_\mathcal{G},\theta_\mathcal{E}}CE(\mathcal{C}(\hat{\mathbf{x}}),\hat{y})-\eta_1\cdot log(1+exp(\mathbf{D}(\hat{\mathbf{x}})))-\eta_2\cdot \frac{1}{d_w}\sum_{j=0}^{d_w-1}\sigma(\{\mathbf{w}^j\}_N),
	\end{equation}
 where $d_w$ represents the dimensionality of $\mathbf{w}$ and $\mathbf{w}^j$ denotes the value in the $j$th dimensionality of the latent feature ($\mathbf{w}=\mathbf{E}(\mathbf{x}^\prime_q)$).
In (\ref{eq:cloudatt}), the first term is the CE loss, meaning that the predicted class label of $\hat{\mathbf{x}}$ is expected to be the same as that of the ciphertext $\mathbf{x^\prime_q}$. The second term, with $\eta_1$ as the coefficient, guides the generated image $\hat{\mathbf{x}}$ to be perceived as a genuine image by a pre-trained discriminator $\mathbf{D}$. For example, as for face images, we adopt the officially pre-trained discriminator proposed by StyleALAE\cite{pidhorskyi2020adversarial}, which demonstrates good performance in generating face images and can effectively classify genuine or synthesized face images. The last term, devised by us, penalizes the average variance of each dimensionality in the latent space produced by $\mathbf{E}$. Such a regularization term helps promote the diversity of the generated content. The parameter $\sigma$ in this term is defined as:
 \begin{equation}\label{eq:cloudatt2} 
\sigma(\{\mathbf{w}^j\}_N)=\frac{\sum_{i=0}^{N-1}(\mathbf{w}_i^j-\frac{1}{N}\sum_{i=0}^{N-1}\mathbf{w}_i^j)^2}{N-1},
	\end{equation}
 where $\mathbf{w}_i^j$ denotes the value in the $j$th dimensionality of the $i$th latent feature ($\mathbf{w}_i=\mathbf{E}(\mathbf{x}^\prime_i)$).

We also build a training dataset that consists of 4500 ciphertexts randomly sampled from VGGFace2 and encrypted by our RIC. The optimization of (\ref{eq:cloudatt}) can then be conducted by using a customized Adam optimizer.

\end{document}